\documentclass{article}
\usepackage{times}

\usepackage{adjustbox}
\usepackage{amsmath, amssymb}
\usepackage{algorithm}
\usepackage{algpseudocode}
\usepackage{mathtools}

\usepackage{amsmath,amsfonts,bm}

\def\Figref#1{Figure~\ref{#1}}

\def\Secref#1{Section~\ref{#1}}

\def\eqref#1{equation~\ref{#1}}
\def\Eqref#1{Equation~\ref{#1}}

\def\1{\bm{1}}

\DeclareMathAlphabet{\mathsfit}{\encodingdefault}{\sfdefault}{m}{sl}
\SetMathAlphabet{\mathsfit}{bold}{\encodingdefault}{\sfdefault}{bx}{n}

\DeclareMathOperator{\sign}{sign}

\usepackage{url}
\usepackage{hyperref}       
\hypersetup{
	colorlinks   = true, 
	urlcolor     = blue, 
	linkcolor    = blue, 
	citecolor   = black 
}
\usepackage{url}            
\usepackage{booktabs}       
\usepackage{amsfonts}       
\usepackage{nicefrac}       
\usepackage{microtype}      
\usepackage{xcolor}         
\usepackage{fullpage}

\usepackage[english]{babel} 
\usepackage[numbers]{natbib}
\usepackage{enumerate}
\usepackage{bbm}
\usepackage{dsfont}
\usepackage{mathtools}

\usepackage{float}
\usepackage{amssymb}
\usepackage{amsmath}
\usepackage{amsthm}

\usepackage{enumitem}
\usepackage{caption}
\usepackage{multirow}
\usepackage{graphicx}
\usepackage{subcaption}

\newtheorem{theorem}{Theorem}
\newtheorem{lemma}[theorem]{Lemma}
\newtheorem{definition}{Definition}
\newtheorem{remark}{Remark}

\def\int{\displaystyle\mathop {\mbox{\rm int}}}    

\DeclareUnicodeCharacter{2212}{\ensuremath{-}}

\newcommand{\relu}[1]{\left[ #1 \right]_+}

\newcommand{\bx}{\mathbf{x}}
\newcommand{\bw}{\mathbf{w}}

\newcommand{\bu}{\mathbf{u}}

\newcommand{\bz}{\mathbf{z}}
\newcommand{\bc}{\mathbf{c}}

\newcommand{\bnu}{\boldsymbol{\nu}}

\newcommand{\btheta}{\boldsymbol{\theta}}

\newcommand{\bSigma}{\boldsymbol{\Sigma}}

\def\real{\mathbb R}

\usepackage{tasks}

\newcommand{\bbs}{{\bf S}}

\newcommand{\bbu}{{\bf U}}

\newcommand{\bbv}{{\bf V}}
\newcommand{\bbx}{{\bf X}}

\newcommand{\bba}{{\bf A}}

\newcommand{\norm}[1]{\left\|#1\right\|}

\def\real{\mathbb R}

\DeclareMathOperator{\spn}{span}

\title{No Prior, No Leakage:\\ Revisiting Reconstruction Attacks in Trained Neural Networks}

\author{Yehonatan Refael, Guy Smorodinsky,  \thanks{ Use footnote for providing further information
about author (webpage, alternative address)---\emph{not} for acknowledging
funding agencies.  Funding acknowledgements go at the end of the paper.} \\
Department of Computer Science\\
Cranberry-Lemon University\\
Pittsburgh, PA 15213, USA \\
\texttt{\{hippo,brain,jen\}@cs.cranberry-lemon.edu} \\
\And
Guy Smorodinsky \& Yevgeny LeNet \\
Department of Computational Neuroscience \\
University of the Witwatersrand \\
Joburg, South Africa \\
\texttt{\{robot,net\}@wits.ac.za} \\
\AND
Coauthor \\
Affiliation \\
Address \\
\texttt{email}
}

\author{
\textbf{Yehonathan Refael}\textsuperscript{1}\thanks{Equal contribution.}
\quad
\textbf{Guy Smorodinsky}\textsuperscript{2}\footnotemark[1]
\quad
\textbf{Ofir Lindenbaum}\textsuperscript{3}
\quad
\textbf{Itay Safran}\textsuperscript{2}\\[2pt]
\textsuperscript{1}Tel Aviv University \quad
\textsuperscript{2}Ben-Gurion University of the Negev \quad
\textsuperscript{3}Bar-Ilan University
}

\usepackage{enumitem}

\begin{document}

\maketitle

\begin{abstract}
The memorization of training data by neural networks raises pressing concerns for privacy and security. Recent work has shown that, under certain conditions, portions of the training set can be reconstructed directly from model parameters. Some of these methods exploit implicit bias toward margin maximization, suggesting that properties often regarded as beneficial for generalization may actually compromise privacy. Yet despite striking empirical demonstrations, the reliability of these attacks remains poorly understood and lacks a solid theoretical foundation. In this work, we take a complementary perspective: rather than designing stronger attacks, we analyze the inherent weaknesses and limitations of existing reconstruction methods and identify conditions under which they fail. We rigorously prove that, without incorporating prior knowledge about the data, there exist infinitely many alternative solutions that may lie arbitrarily far from the true training set, rendering reconstruction fundamentally unreliable. Empirically, we further demonstrate that exact duplication of training examples occurs only by chance. Our results refine the theoretical understanding of when training set leakage is possible and offer new insights into mitigating reconstruction attacks. Remarkably, we demonstrate that networks trained more extensively, and therefore satisfying implicit bias conditions more strongly -- are, in fact, less susceptible to reconstruction attacks, reconciling privacy with the need for strong generalization in this setting.

\end{abstract}

\section{Introduction}

Neural networks have achieved remarkable success across a wide variety of tasks, but their use raises fundamental concerns of privacy \citep{deep_learning_with_DP_OG_paper,runkel2024,fang2024privacy,tan2024defending,bombari2025privacy}. Recent work has demonstrated that portions of the training data can be reconstructed directly from the parameters of a trained model, even without access to gradients or queries \citep{haim2022reconstructing,buzaglo2023deconstructing,loo2023understanding}. Unlike some of the previous methods that produced generic reconstructions resembling class prototypes or averages \citep{carlini2021extracting,carlini2019secret}, the new techniques generate highly accurate and specific reproductions of the original training data.
Such reconstruction attacks highlight the risk that sensitive or private information may leak from models, undermining their safe deployment in practice.

Despite the alarming success of these attacks, our theoretical understanding of them remains limited. 
In particular, the attack introduced in \citep{haim2022reconstructing}, along with some subsequent work inspired by it, such as \citep{buzaglo2023deconstructing}, builds on results concerning the implicit bias of gradient-based optimization in training homogeneous networks \citep{Lyu2020Gradient, KKT2020}.

Intuitively, when optimization succeeds under certain assumptions, such networks trained using standard techniques do not merely converge to any solution that fits the training set; instead, they converge to specific solutions that satisfy additional constraints. Building on this observation, \citet{haim2022reconstructing} constructs an objective function that is minimized when these constraints are satisfied and carries out their attack by optimizing it. Nevertheless, despite its theoretical foundations, our understanding of the conditions under which the attack succeeds remains rudimentary, as rigorous analyses of this setting are scarce.

Recently, \citet{smorodinsky2025provableprivacyattackstrained} provided rigorous guarantees on the efficacy of such a reconstruction attack. However, these guarantees rely on restrictive assumptions, such as a univariate data distribution, which may limit their practical applicability. This leaves open many fundamental questions and motivates several follow-up directions aimed at establishing clear conditions under which such attacks can either be provably effective or provably mitigated. 
A central question that arises is: 
\begin{quote}
    \textit{To what extent do the constraints imposed by the implicit bias of trained neural networks leak information about the training data?}
\end{quote}

In this paper, we address the above question by rigorously studying properties of the objective function used to fuel the reconstruction attack introduced by \citet{haim2022reconstructing}. We propose a method that, given the original training data, enables the construction of multiple other global minima for the objective function used in the attack. Additionally, we demonstrate that under certain conditions, these global minima are ubiquitous, and that solutions to this optimization problem exist that are substantially different from the original training data.\footnote{It is noteworthy that this result answers a question which was raised in \citet{haim2022reconstructing}: ``On the theoretical side, it is not entirely clear why our optimization problem in \Eqref{eq:full_recons_loss} converges to actual training samples, even though there is no guarantee that the solution is unique, especially when using no prior." -- solutions are ubiquitous rather than unique, and in the absence of prior knowledge, the attack fails.} Moreover, we show that the closer the network is to a solution, the harder it becomes to successfully execute such attacks, perhaps contrary to previous common wisdom. Lastly, we empirically demonstrate that if an attacker initializes the attack far from the true training data, the reconstruction obtained contains instances that are significantly farther away from the training data as well. This highlights that additional \emph{prior knowledge} (as was used by \citet{haim2022reconstructing}) is crucial for conducting successful reconstruction attacks in this setting, and that the constraints imposed by the implicit bias alone do not necessarily leak information about the training data.

The rest of this paper is structured as follows: After specifying our contributions in detail below, we turn to discuss additional related work. In \Secref{sec:background}, we present the required background and notation used throughout this paper. In \Secref{sec:real-kkt} we theoretically study the set of feasible solutions to the objective function used in implicit-bias-driven privacy attacks, and in \Secref{sec:experiments} we empirically validate and support our theoretical findings. Lastly, \Secref{sec:summary} summarizes our contributions, and discusses potential limitations and future work directions.

\paragraph{Our contribution.} We prove that the attack presented by \citet{haim2022reconstructing} has inherent limitations. In particular, we demonstrate that prior knowledge about the data domain is necessary to accurately recover the true training examples, and we empirically show several scenarios in which these attacks fail. More specifically:
\begin{itemize}
    \item In Subsection~\ref{sec:exact_kkt}, we demonstrate that in the reconstruction attack devised by \citet{haim2022reconstructing}, there are numerous potential candidate training sets that seem indistinguishable from each other. We provide simple, constructive techniques for generating alternative candidates, either by merging two data instances into one (Lemma~\ref{lma:unite_points_bias}) or by splitting a single instance into two (Lemma~\ref{lma:split_points_bias}). Moreover, we demonstrate that under the assumption that the training data does not span the entire domain, there exist infinitely many such alternative training sets whose distance from the true training set can be arbitrarily large (Theorem~\ref{thm:subspace}). Lastly, we relax this assumption, assuming the training data only approximately lies on a linear subspace, and we further analyze how far a point can be split (Theorem~\ref{thm:kkt_maximal_distance}).

    \item In Subsection~\ref{sec:almost-kkt}, we consider the more realistic setting in which the trained model only approximately satisfies the implicit bias constraints, and present merging and splitting techniques analogous to those mentioned in the previous bullet point (Lemmas~\ref{lma:unite_points_bias_almost_kkt} and~\ref{lma:split_points_bias_almost_kkt}). We then assume that the attacker has limited knowledge regarding the proximity to an implicit bias solution and analyze the extent to which individual points can be split (Theorem~\ref{thm:distance_almost_kkt}). Finally, we demonstrate that even if the attacker possesses additional information about the training procedure, under the assumptions of structured data and a well-trained model, it remains possible to split the points in a way that preserves data confidentiality (Theorem~\ref{thm:delta_almost_kkt}).

    \item In \Secref{sec:experiments}, we complement our theoretical results with experiments that support our findings. Specifically, we model the attacker’s prior as knowledge of the data domain boundaries, incorporated into the reconstruction attack through the attacker’s initialization. We demonstrate, on both synthetic data and CIFAR, that as the attacker’s prior weakens, the effectiveness of the attack decreases accordingly, with convergence to solutions far from the true training set, as predicted by our theory.
\end{itemize}

\subsection*{Additional related work} \label{sec:related_work}

\paragraph{Reconstruction-based privacy attacks.} An emerging line of work studies the possibility of extracting training data directly from trained models, raising serious privacy concerns. Such reconstruction attempts have been demonstrated across a range of settings, including large generative models \citep{carlini2019secret,carlini2021extracting,nasr2023scalable}, diffusion models and diffusion architectures \citep{somepalli2022diffusion,carlini2023extracting}, and federated learning \citep{zhu2019deep,he2019model,hitaj2017deep,geiping2020inverting,huang2021evaluating,wen2022fishing} scenarios. Several studies, most notably \cite{haim2022reconstructing,buzaglo2023reconstructingtrainingdatamulticlass}, have proposed optimization-based strategies that leverage the implicit bias of neural networks to recover examples from the original training data. Their approach frames data recovery as minimizing a suitably defined objective whose solution can align with portions of the true training set.

A complementary line of work, exemplified by \citet{loo2023understanding}, studies attacks on fine-tuned models to reconstruct downstream private training data, in contrast to the setting of \citet{haim2022reconstructing}, which we also examine in this paper. This scenario is particularly important, given the widespread use of fine-tuning on public models (e.g., LLaMA, DeepSeek, Mistral). Nonetheless, the theoretical guarantees presented by \citet{loo2023understanding} are established under simplified and somewhat unrealistic settings that do not fully capture practical conditions, as the underlying proof assumes a particular infinite-width network model and data that lie exactly on the unit hypersphere.

A different perspective investigates the reliability of reconstruction attacks. The study in \cite{runkel2024} empirically finds that reconstruction is highly sensitive to initialization, which can result in the generation of plausible samples not present in the original training data. This ambiguity makes it difficult for an adversary to verify whether a recovered image is an actual training sample. While their empirical findings are similar to ours, our work provides rigorous theoretical guarantees that underpin this phenomenon, whereas their study remains purely empirical.

\paragraph{Differential Privacy.}
The current “golden standard” for privacy is the differential privacy framework \citet{Original_DP_paper, DworkRoth2014}. Intuitively, differential privacy quantifies how much a model’s output can change when a single data instance is modified, thereby providing a worst-case guarantee over all neighboring datasets. The smaller this change, the stronger the privacy guarantee. In practice, mechanisms enforce differential privacy by injecting carefully calibrated noise into data, queries, or training procedures \citep{Original_DP_paper,deep_learning_with_DP_OG_paper}. The composition and accounting results quantify the cumulative privacy loss \citep{Kairouz2015Composition,Balle2018Subsampled}. These guarantees often come with utility trade-offs, which are studied in both ERM and deep learning \citep{Chaudhuri2011}. Unlike differential privacy, our approach does not rely on noise injection. Instead, we seek to better understand the underlying causes of implicit-bias-driven privacy vulnerabilities, so that they can be circumvented when possible.

\section{Preliminaries, background and notation} \label{sec:background}

\paragraph{Notation and terminology.}

We consider binary classification with data \((\bx,y)\in\mathbb{R}^d\times\{-1,1\}\) and training set \(\{(\bx_i,y_i)\}_{i=1}^n\). Let \(\Phi(\btheta;\cdot):\mathbb{R}^d\to\mathbb{R}\) be a neural network with parameters \(\btheta\in\mathbb{R}^k\). Write \(\relu{x}\coloneqq\max(0,x)\). A homogeneous 2-layer ReLU network is
$\Phi(\btheta,\bx)=\sum_{j=1}^k v_j\,\relu{\bw_j^\top \bx+b_j},$
where \(\btheta=(\bw_j,v_j,b_j)_{j=1}^k\). The unit sphere is \(\mathbb{S}^{d-1}\coloneqq\{\bx\in\mathbb{R}^d:\|\bx\|_2=1\}\). Let \(d(\bx_1,\bx_2)\) denote the distance between vectors (Euclidean by default), and for sets \(A,B\) define \(d(A,B)\coloneqq \min_{a\in A,b\in B} d(a,b)\). For \(\bx\in\mathbb{R}^d\), its activation pattern is the binary vector whose \(j\)-th entry indicates whether neuron \(j\) is active on \(\bx\). For neuron \((\bw_j,b_j)\) and point \(\bx\), set
$D_j(\bx)\coloneqq\frac{\langle\bw_j,\bx\rangle+b_j}{\|\bw_j\|},$
the signed distance to the hyperplane $\langle\bw_j,\bx\rangle+b_j=0$. 

\paragraph{Preliminaries.}

The following well-established result, which has recently attracted significant attention, shows that homogeneous neural networks trained with logistic or exponential loss exhibit an implicit bias: they converge in direction to the solution of a max-margin problem. This implicit bias is the key phenomenon underlying the attacks we study in this paper.

\begin{theorem}[paraphrased version of \citet{Lyu2020Gradient}, \citet{KKT2020}]\label{thm:implicit_bias}
    Let $\Phi(\btheta; \bx)$ be a normalized homogeneous\footnote{A network $\Phi(\btheta; \bx)$ is called \emph{homogeneous} if there exists $c > 0$ such that for every $b > 0$, $\btheta$ and $\bx$, it holds that $\Phi(b \cdot \btheta; \bx) = b^c\Phi(\btheta; \bx)$ \cite{}.} ReLU neural network. 
    Consider minimizing the logistic ($z\mapsto \log(1+e^{-z})$) or the exponential ($z\mapsto e^{-z}$) loss using gradient flow (which is a continuous time analog of gradient descent) over a binary classification set $\{(\bx_i, y_i)\}_{i=1}^n \subseteq \real^d \times \{-1 ,1\}$. Assume that there is a time $t_0$ where $L(\btheta(t_0)) < \frac{1}{n}$. Then, gradient flow converges in direction\footnote{We say that gradient flow \emph{converges in direction} to $\hat{\btheta}$ if $\lim_{t \rightarrow \infty}\frac{\btheta(t)}{\|\btheta(t)\|} = \frac{\hat{\btheta}}{\|\hat{\btheta}\|}$.} to a first order stationary point (KKT point \cite{vapnik1995support}) of the following maximum-margin problem:
    \begin{equation*}\label{eq:maximal_margin}
      \min_{\btheta} \frac{1}{2} \|\btheta\|^2 ~~\text{s.t} ~~\forall i \in [n] ~~ y_i \Phi(\btheta;\bx_i) \geq 1.
    \end{equation*}
\end{theorem}
A KKT point of \Eqref{eq:maximal_margin} is characterized by the following set of conditions:
\begin{align}
&\btheta=\sum_{i=1}^n \lambda_i \nabla_{\btheta}\left[y_i \Phi\left(\btheta ; \bx_i\right)\right] & \text { (stationarity) }\label{eq:stationary} \\
&y_i \Phi\left(\btheta ; \bx_i\right) \geq 1, \forall i \in[n] & \text { (primal feasibility) }\label{eq:margin} \\
& \lambda_i \geq 0, \forall i \in[n] & \text { (dual feasibility) }\label{eq:positive_lambda} \\
& \lambda_i=0 \text { if } y_i \Phi\left(\btheta ; \bx_i\right) \neq 1, \forall i \in[n] & \text { (complementary slackness) } \label{eq:zero_lam}
\end{align}

Utilizing the above result, \citet{haim2022reconstructing} devised the following reconstruction attack, demonstrating that one can reconstruct a substantial subset of the training data from a deployed model.

The core of the method is a minimization problem. It simultaneously adjusts the candidate training data $\bbx^\prime$ and their associated Lagrange multipliers $\lambda$ to minimize a composite loss function. This loss function is designed to enforce several of the KKT conditions that characterize an optimal solution for the classifier. The objective function is given explicitly by
\begin{equation}\label{eq:full_recons_loss}
    \bbx^\prime = \arg \min_{\{\lambda_i, x'_i, i \in [n]\}} \underbrace{\left\| \theta - \sum_{i=1}^m \lambda_i \nabla_{\theta} [y_i \Phi(\theta; x'_i)] \right\|}_{L_{\mathrm{stationary}}} + \underbrace{\sum_{i=1}^m \max\{-\lambda_i, 0\}}_{L_{\lambda}} + L_{\mathrm{prior}}(X').
\end{equation}

The stationary loss in \Eqref{eq:full_recons_loss} penalizes deviations from the KKT stationary condition (\Eqref{eq:stationary}), where the non-negativity loss enforces the non-negativity constraint on the multipliers $\lambda_i$, as required by \Eqref{eq:positive_lambda}, and $L_{\mathrm{prior}}$ represents some prior knowledge we might have about the dataset.
Lastly, note that the margin's value, as defined by the primal feasibility condition in \Eqref{eq:margin} at a KKT point, is normalized \citep{Lyu2020Gradient} and thus equals $1$. However, in practice, the margin width $\gamma(\btheta)$ may take on a different scale, which, for later use, we denote by $\gamma(\btheta) = p>0$.\footnote{The margin $\gamma(\btheta)$ scales linearly
with $\|\btheta\|_{2}^{c}$, where $c$ is the order of homogeneity, namely $\tilde{\gamma}(\btheta) := 
\frac{\gamma(\btheta)}{\|\btheta\|_{2}^{c}}$.} We remark that this method aims to reconstruct training data instances that are on the margin (namely, instances whose prediction value is $p$ or $-p$), since otherwise, by \Eqref{eq:zero_lam}, they do not factor in the constructed objective.

Although incorporating prior knowledge as an additional loss term to $L_{\mathrm{KKT}}$ is a common technique, we deliberately exclude it from our analysis. Our rationale is that such priors are often heuristic and highly context-specific, rather than fundamental to the reconstruction method in \Eqref{eq:full_recons_loss}.
Including a prior would confound the evaluation of the method's performance, as success would become dependent on the specificity of the external knowledge, making it difficult to disentangle the contribution of the mathematical constraints imposed by Equations~\ref{eq:stationary}-\ref{eq:zero_lam}.

\section{Implicit-bias-based reconstruction attacks}\label{sec:real-kkt}

In this section, we theoretically study the reconstruction attack introduced by \citet{haim2022reconstructing}.
We begin by formally defining our framework and the required definitions, which differ slightly from those in the practical setting studied in their seminal work.

\paragraph{Framework.} An attacker, motivated to reconstruct the training set $S$, is given a full access to the neural network (binary classifier) architecture $\Phi(\btheta;\cdot),$ as well as complete knowledge of the weights $\btheta$.\footnote{Many of our results can still hold, or be extended, with only partial access to the network’s weights. For clarity and simplicity, however, we assume full access in this work.} However, the attacker has no knowledge of the training samples $\bbx=(\bx_1,\ldots,\bx_n)$ or the quantity $n$. In addition, we first assume no knowledge of the margin scale $\gamma(\btheta) = p$ in the following subsection, and in Subsection~\ref{sec:almost-kkt} we relax this assumption. With this information, the attacker then optimizes the objective in \Eqref{eq:full_recons_loss}
to reconstruct the training set, yielding $\bbx^\prime=(\bx^\prime_1,\ldots,\bx^\prime_n)$.

We begin by defining the KKT loss, which is used to assess the feasibility of a reconstructed set. As mentioned in the previous section, this definition omits the prior term.

\begin{definition}[KKT-loss]\label{def::kkt_loss} 
The \emph{KKT-loss} is defined as,
\begin{align}
L_{\mathrm{KKT}}(\bx_1, \dots, \bx_l, \lambda_1, \dots, \lambda_l) \coloneqq \gamma_1 L_{\mathrm{stationary}}(\bx_1, \dots, \bx_l) + \gamma_2 L_{\lambda}(\lambda_1, \dots, \lambda_l), \label{Eq::kkt_loss} 
\end{align} 
where $\gamma_1, \gamma_2 > 0$ denote hyperparameters controlling the optimization process.
\end{definition}

\begin{remark}
    Note that $l$, the number of data samples in Definition~\ref{def::kkt_loss}, does not necessarily equal $n$. This is justified, as by assumption, the attacker does not know the number of training samples $n$. Additionally, we do not incorporate the constraints in \Eqref{eq:margin} and \Eqref{eq:zero_lam} into the definition of $L_{\mathrm{KKT}}$, as these were also not used in \citet{haim2022reconstructing}.
\end{remark}

Following the definition of our loss function, we specify the corresponding set of examples that satisfy the conditions of Theorem~\ref{thm:implicit_bias}, which
the minimization of the objective in \Eqref{Eq::kkt_loss} aims to find.

\begin{definition}[KKT set]
    Let $\Phi(\btheta; \cdot)$ be a binary classification network with weights $\btheta$ that has converged to a KKT point (satisfying Equations~\ref{eq:stationary}--\ref{eq:zero_lam}).  
    A set of inputs $S = \{\bx_1, \dots, \bx_l\}$ is called a \emph{KKT set} if there exist nonnegative multipliers $\lambda_1, \dots, \lambda_l \geq 0$ such that $L_{\mathrm{KKT}}(S, \lambda_1, \dots, \lambda_l) = 0$.
\end{definition}

Clearly, the original training set $S$ is a KKT set by assumption.

\subsection{Convergence to an exact KKT point}\label{sec:exact_kkt}

We now show that, perhaps surprisingly, once the prior component is removed from \Eqref{eq:full_recons_loss}, the objective $L_{\mathrm{KKT}}$ admits infinitely many global minima. Moreover, certain minima lead to reconstructed samples that differ substantially from the original training set, as measured by the minimal Euclidean distance between neighboring instances. Consequently, the reconstruction attack cannot reliably distinguish the actual training set from these alternative KKT sets. Remarkably, we further show that the distance between such minima and the original training set can be unbounded.

To this end, we present a constructive method for generating new KKT sets from a given one. This method relies on two key lemmas that underpin the construction, after which we demonstrate how they can be applied to generate a broad family of KKT sets explicitly.

The proofs of lemmas and theorems in this subsection are relegated to Appendix~\ref{sec:exact_kkt_proofs}.

\begin{lemma}[Merge] \label{lma:unite_points_bias}
    Let $S$ be a KKT set and let $\bx_1, \bx_2 \in S$ be two points with identical labels and activation patterns, and with coefficients $\lambda_1, \lambda_2 > 0$. Then, there exists $\alpha \in (0,1)$ such that the set $S^\prime \coloneqq (S \setminus \{\bx_1, \bx_2\}) \cup \{\bx_{1.5}\}$ is also a KKT set, where $\bx_{1.5} = \alpha \bx_1 + (1-\alpha)\bx_2$.
\end{lemma}

\begin{lemma}[Split]\label{lma:split_points_bias}
    Let $S$ be a KKT set and let $\bx_1 \in S$ be a point with coefficient $\lambda_1 > 0$. Then, for all $\alpha, \beta > 0$ and for all $\bnu \in \real^d$ such that 
    $\bz_1 = \bx_1 + \alpha \bnu$ and $\bz_2 = \bx_1 - \beta \bnu$ have the same activation pattern and classification as $\bx_1$, 
    the set  $S^\prime \coloneqq (S \setminus \{\bx_1\}) \cup \{\bz_1, \bz_2\}$ is a KKT set.
\end{lemma}

The above lemmas offer a constructive approach for discovering new global minima of the objective in \Eqref{Eq::kkt_loss}. While the former allows merging two points to create a new KKT set, the latter enables splitting a single point into two. A more visual illustration of this concept can be found in Figure~\ref{fig:unite_split_examples}.

\begin{figure}[!b]
    \centering
    \vskip -20pt\begin{subfigure}[t]{0.48\textwidth}
        \centering
        \includegraphics[width=0.65\textwidth]{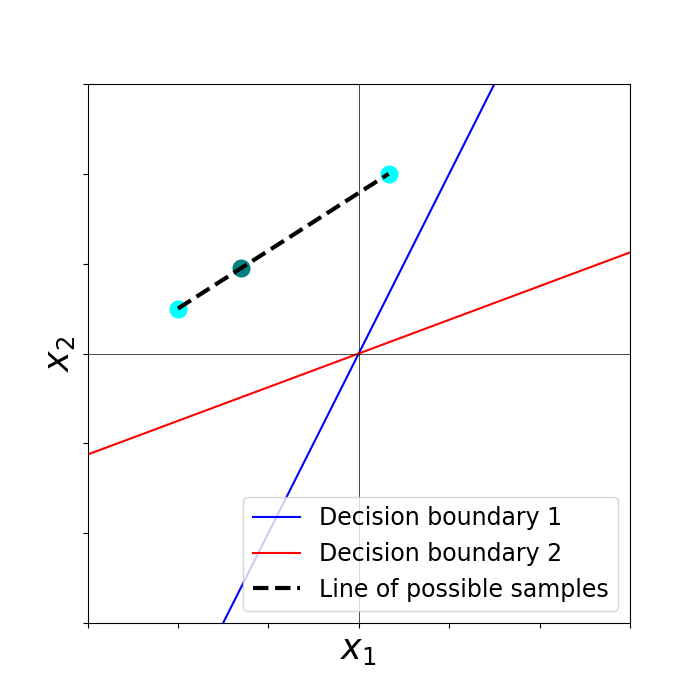}
        \caption{Example of Lemma~\ref{lma:unite_points_bias}. The green point on the dotted line can replace the two points at its edges.}
        \label{fig:unite_example}
    \end{subfigure}
    \hfill
    \begin{subfigure}[t]{0.48\textwidth}
        \centering
        \includegraphics[width=0.65\textwidth]{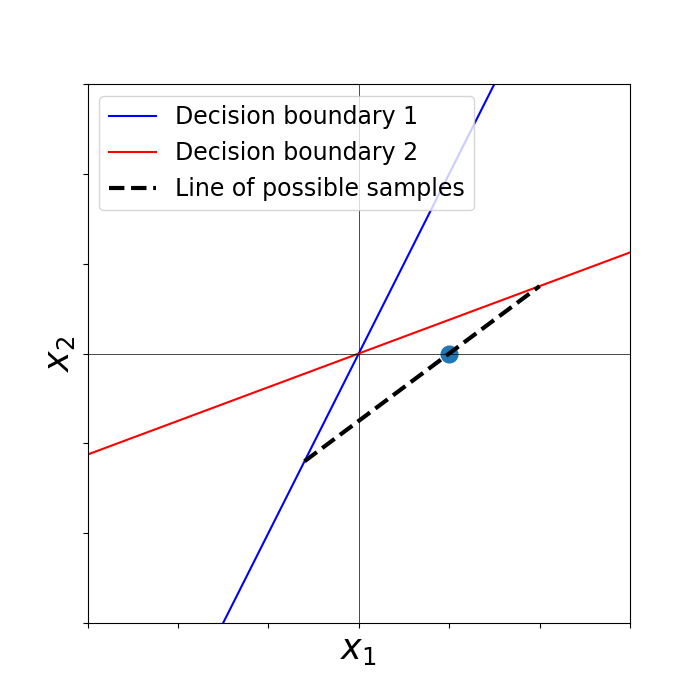}
        \caption{Example of Lemma~\ref{lma:split_points_bias}. The blue point can be split into two points along the dotted line.}
        \label{fig:split_example}
    \end{subfigure}
\caption{Illustration of Lemmas~\ref{lma:unite_points_bias}~and~\ref{lma:split_points_bias} which demonstrate how splitting and merging can allow us to constructively explore the solution space of the attacker's optimization problem, presented in Definition~\ref{def::kkt_loss}. An instance can be split into two instances along a valid direction (Subfigure~\ref{fig:unite_example}) or, conversely, two instances can be combined into a single representative instance (Subfigure~\ref{fig:split_example}). Note that the instances obtained through splitting or merging preserve the same activation pattern as the original instance(s).}
\label{fig:unite_split_examples}
\end{figure}

While these technical results provide means to explore the loss surface of the objective in \Eqref{Eq::kkt_loss}, they do not provide a quantifiable assessment of the extent to which KKT sets can be manipulated by merging and splitting. While merging imposes few constraints, it provides more limited alterations to the dataset. In contrast, splitting allows flexible alterations, but further requires that the new points must have the activation pattern and classification of the original point from which they are derived, which may constrain their distance from it. Conveniently, under the assumption that the set of training examples does not span the entire data domain, we can show that this distance is unbounded.

\begin{theorem}\label{thm:subspace}
    Let $S = \{\bx_1,\dots,\bx_n\} \subset \real^d$ be the training set, and let $\Phi(\btheta; \bx)$ a 2-layer neural network that was trained on $S$ and reached a KKT point. If $\spn\{\bx_1,\dots,\bx_n\} \subsetneq \real^d$, then for all $r > 0$, there exists a KKT set $S_r$ such that $d(S, S_r) > r$.
    Moreover, all points in $S_r$ are on the margin. That is, $|\Phi(\btheta;\bx_r)| = p$ for all $\bx_r \in S_r$.
\end{theorem}

Interestingly, the above theorem implies that the KKT sets that we construct provide a twofold defense: not only are they global minima of the KKT loss, but all points in these sets also lie on the margin. Consequently, even if the margin value $p$ is leaked, the attacker cannot use it to distinguish the actual training set from an arbitrary KKT set.

We stress that, due to Equations~\ref{eq:stationary} and~\ref{eq:zero_lam}, the above theorem pertains only to points on the margin. Thus, as with all other results in this subsection, it continues to hold even if arbitrarily many non-margin points are added to $S$. However, for the sake of simplicity, we assume here that all points in $S$ are on the margin.

The assumption that the training examples do not span the entire data domain can be justified by certain real-world datasets, which concentrate on low-dimensional structures or even unions of subspaces within the entire domain, a property commonly exploited in practice \citep{elhamifar2013sparse}. For example, this is evident in the MNIST dataset, where most images consist primarily of black pixels. Still, in many cases, the data manifold is not a proper subspace of the domain, but instead can be closely approximated by such a subspace. To provide meaningful guarantees in such cases, we present the following theorem.

\begin{theorem}\label{thm:kkt_maximal_distance}
    Let $S = \{\bx_1,\dots,\bx_n\} \subset \real^d$ be a KKT set, let $\Phi(\btheta; \bx)$ be a 2-layer, width-$k$ neural network that was trained on $S$ and reached a KKT point, and suppose that $\gamma>0$ and $\bnu \in \real^d$ satisfy $\langle\bnu,\bx_i\rangle\le\gamma$ for all $i$.
    Then, for all $\bx_i \in S$ and any $\alpha, \beta$ not exceeding
    \[
       \min_{j \in [k]} \frac{|D_j(\bx_i)| \cdot \|\bw_j\|}{\sum_{i=1}^n \lambda_i} \cdot \frac{1}{\gamma},
    \]
    $S'\coloneqq (S \setminus \{\bx_1\}) \cup \{\bx_i + \alpha \bnu, \bx_i - \beta \bnu\}$ is a KKT set.
\end{theorem}
In words, the theorem implies that each point $\bx_i$ can be split into two points, each at least $\min_{j \in [k]} \frac{|D_j(\bx_i)| \cdot \|\bw_j\|}{\sum_{i=1}^n \lambda_i} \cdot \frac{1}{\gamma}$ away from $\bx_i$ (larger values of $\alpha$ and $\beta$ may also be admissible). Note that $d \leq n$, since otherwise one could find a vector $\bnu$ orthogonal to all data points, giving $\gamma = 0$ and reducing the setting to the assumptions of Theorem~\ref{thm:subspace}. We stress that smaller values of $\gamma$ allow points to be split farther -- a situation that arises, for example, when the data manifold approximately lies on a linear subspace of the domain. Moreover, we prove that $\gamma \le \sigma_d$, where $\sigma_d$ is the smallest singular value of the data matrix $S$ written as a column matrix $\bbs = (\bx_1, \dots, \bx_n)$ (see Appendix~\ref{app:SVD} for further discussion). Geometrically, this means that the minimal splitting distance is controlled by the ``thinnest" direction of the data cloud: the smaller the smallest singular value, the more freedom there is to alter points along directions perpendicular to those where the data is nearly flat.

\subsection{Analysis beyond the idealized setting: approximate KKT points}\label{sec:almost-kkt}

In the previous subsection, we assumed that the model converged to a KKT point. Since, in practice, a trained network may only approximate such a stationary point rather than precisely attain it, this subsection provides a more natural yet analytically tractable foundation for studying the associated privacy risks. To that end, we will shortly formulate the following relaxation of the KKT condition.


\begin{definition}[$(\varepsilon, \delta)$-KKT -- paraphrased from \citet{Lyu2020Gradient, KKT2020}] \label{def:almost_kkt}\hfill\\
    $\btheta$ satisfies $(\varepsilon, \delta)$-KKT if the following holds:
    \begin{tasks}(2) 
        \task $\|\btheta - \sum_{i=1}^n \lambda_i y_i \nabla_{\btheta} \Phi(\btheta; \bx_i) \|_2 \leq \varepsilon$.\label{eq:epsilon_kkt}
        \task $\forall i, ~ y_i \Phi(\btheta; \bx_i) \geq p > 0$.\label{eq:almost_margin}
        \task $\lambda_1, \dots, \lambda_n \geq 0$. \label{eq:almost_positive_lambdas}
        \task $\forall i,~ \lambda_i (y_i \Phi(\btheta; \bx_i) - p) \leq \delta$.\label{eq:delta_kkt}
    \end{tasks}
\end{definition}

Following the above definition, the following is a natural generalization of a KKT set.

\begin{definition}[$(\varepsilon, \delta)$-KKT sets]\label{def:epsilon_kkt_set}
    Let $\Phi(\btheta; \bx)$ be a homogeneous neural network, and let $S= \{\bx_1, \dots, \bx_n\}$. We say that $S$ is an $(\varepsilon,\delta)$-KKT set if there exist nonnegative $\lambda_1,\ldots,\lambda_n$ such that $\btheta$ is $(\varepsilon,\delta)$-KKT.
\end{definition}
We note that optimizing the objective in Condition~\ref{eq:epsilon_kkt} to accuracy $\varepsilon$ produces $(\varepsilon,\delta)$-KKT sets for some $\delta>0$. For brevity, when $\delta$ is irrelevant, we simply refer to these as $\varepsilon$-KKT sets. This convention aligns with the setting in \citet{haim2022reconstructing}, who similarly disregard the $\delta$ term, since it is not used in their attack.

The following are extensions of the merging and splitting lemmas from the previous subsection.

\begin{lemma}[Approximate-KKT merge]\label{lma:unite_points_bias_almost_kkt}
    Let $S$ be an $\varepsilon$-KKT set and let $\bx_1,\bx_2 \in S$ be a point with coefficients $\lambda_1, \lambda_2 > 0$. Then, there exists $\alpha \in (0,1)$ such that the set $S^\prime = (S \setminus \{\bx_1,\bx_2\}) \cup \{\alpha \bx_1 + (1-\alpha) \bx_2\}$ is also an $\varepsilon$-KKT set.
\end{lemma}

\begin{lemma}[Approximate-KKT split]\label{lma:split_points_bias_almost_kkt}
    Let $S$ be an $\varepsilon$-KKT set and let $\bx_1 \in S$ be a point with coefficient $\lambda_1 > 0$. 
    Consider the two points: $\bz_1 = \bx_1 + \alpha \bnu$ and $\bz_2 = \bx_1 - \beta \bnu$ where $\alpha,\beta > 0$ and $\bnu \in \real^d$ such that $\bz_1, \bz_2$ have the same activation pattern and classification as $\bx_1$. Then, the set $S^\prime = (S \setminus \{\bx_1\}) \cup  \{\bz_1, \bz_2\}$ is also an $\varepsilon$-KKT set.
\end{lemma}

In the following two subsections, we examine the budget for splitting data points in the almost KKT setting. Each subsection assumes full knowledge of $\varepsilon$ but considers a different level of knowledge the attacker has about $\delta$.

The proofs of lemmas and theorems are relegated to Appendix~\ref{sec:almost_kkt_proofs}.

\subsubsection{The attacker possesses no information about $\delta$} \label{sec:almost_kkt_no_knowledge_delta}

As a starting point, we first assume in this subsection that the attacker has no information about $\delta$. In such a case, Lemma~\ref{lma:split_points_bias_almost_kkt} can still be used as long as the new points are in the same activation pattern as the original point.

The following theorem provides a minimal guarantee on the budget for splitting a data point in this setting.

\begin{theorem} \label{thm:distance_almost_kkt}
    Let $S = \{\bx_1,\dots,\bx_n\} \subset \real^d$ be an $\varepsilon$-KKT set for some $\varepsilon>0$, let $\Phi(\btheta; \bx)$ be a 2-layer, width-$k$ neural network that was trained on $S$ and reached an $(\varepsilon, \delta)$-KKT point, and suppose that $\gamma>0$ and $\bnu \in \real^d$ satisfy $\langle\bnu,\bx_i\rangle\le\gamma$ for all $i$.
    Then, for all $\bx_i \in S$ and any $\alpha, \beta$ not exceeding
    \begin{equation}\label{eq:alpha_beta_budget}
        \min_{j \in [k]} \frac{|D_j(\bx_l)| \|\bw_j\|}      {\varepsilon + \gamma |v_j| \sum_{i=1}^n \lambda_i},
    \end{equation}
    such that $\bz_1\coloneqq\bx_i + \alpha \bnu$ and $\bz_2\coloneqq \bx_i - \beta \bnu$ have the same classification as $\bx_i$, $S'\coloneqq (S \setminus \{\bx_i\}) \cup \{\bz_1,\bz_2\}$ is an $\varepsilon$-KKT set.
\end{theorem}

Note that when $\gamma = 0$, the above equation simplifies to $\min_{j \in [k]} \frac{|D_j(\bx_l)|\|\bw_j\|}{\varepsilon}$. Furthermore, if $\varepsilon \to 0$, then both $\alpha$ and $\beta$ become unbounded, consistent with the analogous case in Theorem~\ref{thm:kkt_maximal_distance}.

\subsubsection{The attacker has an upper bound on $\delta$} \label{sec:sec:almost_kkt_yes_knowledge_delta}

In the previous subsection, we assumed, as an initial test case, that the attacker cannot know $\delta$. This assumption, however, is brittle: in practice, some information about $\delta$ may leak, or the attacker may be able to deduce an upper bound on its value from the training dynamics.\footnote{For instance, a large value of $\delta$ may indicate that the training process terminated prematurely, leading to poor performance.} Because our splitting technique can degrade the value of $\delta$, this upper bound introduces an additional constraint: we cannot split the training set beyond the attacker’s bound, which limits the allowable budget for altering training instances. The following theorem establishes a lower bound on the permissible amount of change in this scenario.

\begin{theorem} \label{thm:delta_almost_kkt}
    Let $S = \{\bx_1,\dots,\bx_n\} \subset \real^d$ be an $(\varepsilon,\delta)$-KKT set for some $\varepsilon,\delta>0$, and let $\Phi(\btheta; \bx)$ be a 2-layer, width-$k$ neural network that was trained on $S$ and reached an $(\varepsilon, \delta)$-KKT point. Suppose that $\gamma>0$ and $\bnu \in \real^d$ satisfy $\langle\bnu,\bx_i\rangle\le\gamma$ for all $i$, and consider the dataset $S^\prime = (S \setminus \{\bx_1\}) \cup  \{\bx_l + \alpha_l\bnu, \bx_l - \beta_l \bnu\}$ for $\alpha$ and $\beta$ not exceeding \Eqref{eq:alpha_beta_budget}. Then, if $\Delta\delta<p$, $S'$ is an $(\varepsilon,\delta+\Delta\delta)$-KKT set, where
    \[
        \Delta\delta\coloneqq \lambda_l (\alpha_l+\beta_l) \sum_{j \in J} |v_j| \left(\varepsilon + \gamma |v_j| \sum_{i=1}^n \lambda_i \right).
    \]
\end{theorem}
When $\gamma = 0$, the expression simplifies to $\delta + \varepsilon \lambda_l(\alpha_l+\beta_l)\sum_{j=1}^k|v_j|$. Thus, although $\delta$ deteriorates with increased splitting, the effectiveness of the reconstruction attack diminishes as $\gamma$ decreases and the network approaches a KKT point, indicating that for well-trained networks on structured data, such attacks become unreliable.

\section{Experiments}\label{sec:experiments}
In this section, we present experiments that complement our theoretical analysis. Our theory demonstrates that the objective in \Eqref{Eq::kkt_loss} admits ubiquitous global minima under certain conditions. However, this alone does not rule out the possibility that the attack converges to the true training set. To investigate this, we empirically test whether training-set leakage can occur without prior knowledge. Specifically, we model the attacker’s prior knowledge as awareness of the data domain boundaries, which we incorporate into the attack by modifying its initialization distribution. For example, when the data domain consists of natural images, this corresponds to prior knowledge that the training set must lie within the range of valid pixel values $[0,1]^d$. Consequently, an attacker with this prior will initialize the attack only with values that respect this constraint. This type of prior was also used by \citet{haim2022reconstructing}, who employed both a correctly scaled initialization and a penalty in the optimization objective to ensure solutions remained within the domain of natural images.

\paragraph{No leakage without prior.}

We generated 500 synthetic training samples (250 per class) uniformly on the unit sphere $\mathbb{S}^{783}\subset\real^{784}$, and labeled them according to the sign of the first coordinate. We then trained a 2-layer width 1,000 ReLU network on this data for 500K epochs, achieving a final training loss of $10^{-7}$. To assess reconstructability, we initialize the candidate reconstructions ${\bx_i}$ on spheres with varying radii centered at the origin (each corresponding to different levels of prior available to the attacker). For each setting, we record the average distance of the 5 best reconstructions to the actual training set across multiple runs. 

Although all runs achieve similar KKT objective values (ranging from 330 to 332), they yield markedly different reconstruction qualities, revealing a strong dependence on initialization and convergence to distinct minima. In the absence of prior information, near-optimal solutions thus do not reliably recover the original samples.
\Figref{fig:avg_distances} depicts the distribution of the average distance between the training set and the best five reconstruction attempts. It demonstrates that the reconstruction error increases as the assumed radius deviates from the true domain of the data. Consequently, successful reconstruction appears to strongly depend on prior knowledge of the data domain.

\paragraph{Beyond the theoretical framework.}
We trained the same 3-layer architecture on CIFAR as was done by \citet{haim2022reconstructing}, after shifting all training samples by various magnitudes, corresponding to different levels of prior information available to the attacker.\footnote{While conceptually equivalent, this is not precisely how we obtained our network. See Appendix~\ref{app:shifting} for further details.} To reconstruct the training set, we minimized the objective in \Eqref{Eq::kkt_loss}, without incorporating any additional regularization. 

As shown in \Figref{fig:move_pixels}, the results quickly deteriorate as the attacker's prior weakens. Moreover, the reconstructions clearly resemble averages of multiple training instances, indicating that the minimum reached by the attack’s optimization procedure is in fact an interpolation of several training examples, as predicted by our theory.

Similarly to the previous synthetic case, our experiments demonstrate that prior knowledge is crucial for reconstruction attacks to succeed, and that the KKT constraints in Equations~\ref{eq:stationary}-\ref{eq:zero_lam} themselves do not necessarily leak information about the training set in the absence of prior knowledge. Moreover, even if knowledge of the data domain is leaked, shifting the training data by a secret bias effectively mitigates privacy risks.

\begin{figure}[h!]
  \centering
  \begin{subfigure}{0.35\linewidth}
    \adjustbox{valign=t}{%
      \begin{minipage}{\linewidth}
        \centering
        \includegraphics[width=1.05\linewidth]{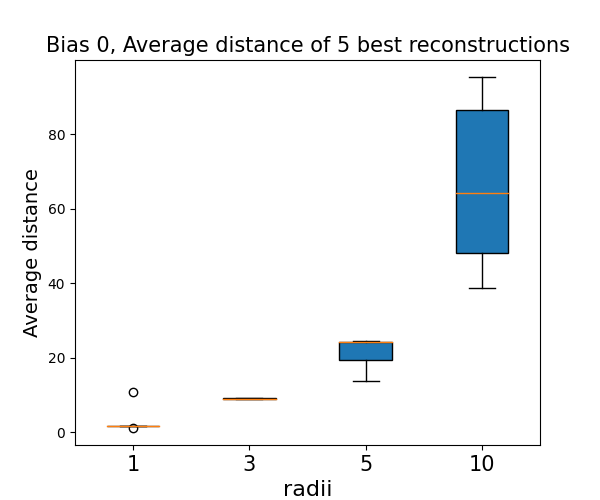}
      \end{minipage}
    }
    \caption{A comparison of the reconstruction results based on initializations with different radii, for a network that was trained with data sampled from the unit sphere, and measured by the average Euclidean distance of the 5 best reconstructions. The radius increases as the prior knowledge available to the attacker weakens, significantly degrading the quality of the reconstruction.}
    \label{fig:avg_distances}
  \end{subfigure}\hfill
  \begin{subfigure}{0.61\linewidth}
    \adjustbox{valign=t}{%
      \begin{minipage}{\linewidth}
        \centering
        \begin{subfigure}{\linewidth}
          \centering
          \includegraphics[width=\linewidth]{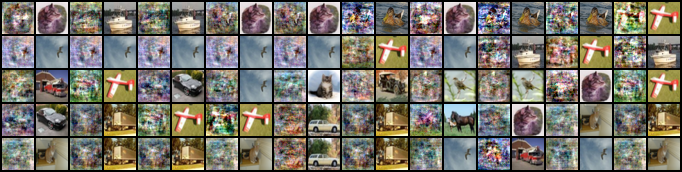}
          \label{fig:top_right}
        \end{subfigure}

        \vspace{0.5em}

        \begin{subfigure}{\linewidth}
          \centering
          \includegraphics[width=\linewidth]{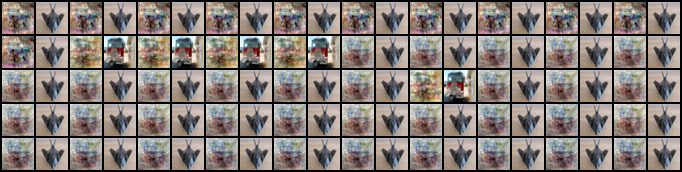}
          \label{fig:bottom_right}
        \end{subfigure}
      \end{minipage}
    }
    \caption{Reconstructed CIFAR images (odd columns) and their nearest neighbors from the training set (even columns). In the top image, all training data pixels were shifted by 0.5, and in the bottom image by 5. We trained a model on the shifted data until it reached an almost-KKT point, without any regularization. The experiment demonstrates that as the data is shifted further, corresponding to a weaker prior available to the attacker, the effectiveness of the attack diminishes rapidly. While the top reconstruction still captures vague characteristics of a small subset of the training set, the bottom reconstruction fails entirely.}
    \label{fig:move_pixels}  
  \end{subfigure}

  \caption{The left figure shows the attack on the unit sphere. The images on the right show the attack on CIFAR with shifted training data.}
\end{figure}

\section{Summary and discussion}\label{sec:summary}

In this paper, we demonstrated both theoretically and empirically that the objective function underlying the reconstruction attack of \citet{haim2022reconstructing} admits ubiquitous global minima. Consequently, reconstruction is generally unreliable without prior knowledge, and this observation suggests new avenues for mitigating such attacks, for example, by shifting the training set with a secret bias. Notably, our results indicate that the implicit bias induced by gradient methods can actually prevent leakage rather than facilitate it, which may seem counterintuitive in light of previous work.

A potential future work direction could study the extent of information leakage, and possible ways to mitigate it in different types of network architectures, such as LLMs \citep{vaswani2023attentionneed}. Such architectures are trained using different, more practical optimizers \citep{refael2025adarankgrad,anonymous2025sumo,zhao2024galorememoryefficientllmtraining}, making their analysis more challenging.

Although our proposed defenses are theoretically motivated, they do not provably preclude reconstruction, since an attacker might still infer information about the data domain directly from the model. We leave the intriguing question of whether this is indeed possible and how to design provably secure defenses for future work.

\subsection*{Acknowledgments}

I.S.\ is supported by Israel Science Foundation Grant No.\ 1753/25.

\bibliography{iclr2026_conference}
\bibliographystyle{abbrvnat}
\clearpage
\appendix


\section{Proofs for Subsection~\ref{sec:exact_kkt}}\label{sec:exact_kkt_proofs}
In this section, we present all the proofs for Subsection~\ref{sec:exact_kkt}. We begin with a few additional notations that will be used throughout the appendix.

Given a neuron with weights $\bw$ and bias $b$, define the shorthand \(c_j(\bx)\coloneqq \relu{\langle\bw_j,\bx\rangle+b_j}\).
Let $\sigma\prime_j$ denote a subgradient of $\relu{\bw_j^\top \bx+b_j}$; for sample \(\bx_i\), write \(\sigma'_{i,j}\) for the subgradient of $\relu{\bw_j^\top \bx_i+b_j}$.

\subsection{Proofs of Lemma~\ref{lma:unite_points_bias} and Lemma~\ref{lma:split_points_bias}}

We begin with proving an auxiliary lemma about the convex nature of the gradient of the neural network, and then we prove Lemmas~\ref{lma:unite_points_bias}~and~\ref{lma:split_points_bias}.

\begin{lemma}\label{lma:gradient-is-convex}
    Let $\Phi(\btheta; \bx)$ be a neural network with piecewise-linear activations. Suppose that the point $\alpha\bx_1 + (1-\alpha)\bx_2)$ is at the same activation pattern as $\bx_1$ and $\bx_2$. Then, $\nabla_{\btheta}\Phi(\btheta; \bx)$ is convex with respect to $\bx$:
    \[
        \nabla_{\btheta}\Phi(\btheta; \alpha\bx_1 + (1-\alpha)\bx_2) = \alpha \nabla_{\btheta}\Phi(\btheta; \bx_1) + (1-\alpha)\nabla_{\btheta}\Phi(\btheta;\bx_2).
    \]
\end{lemma}

\begin{proof}
    Let $\Phi(\btheta; \bx)$ be a neural network with ReLU activations (the argument is the same for other piecewise linear functions), then since the gradient $\nabla_{\btheta}\Phi(\btheta; \bx)$ is with respect to $\btheta$, it remains an affine function in $\bx$ at each activation pattern, and can thus be written as $\nabla_{\btheta}\Phi(\btheta; \bx) = \bba \bx + \bc$ for some matrix $\bba$ and vector $\bc$. We therefore have
    \begin{align*}
        \alpha \nabla_{\btheta}\Phi(\btheta; \bx_1) + (1-\alpha)\nabla_{\btheta}\Phi(\btheta; \bx_2) &= \alpha(\bba \bx_1 + \bc) + (1-\alpha) (\bba\bx_2 + \bc) \\
        &= \bba(\alpha \bx_1 + (1-\alpha) \bx_2) + \bc \\
        &= \nabla_{\btheta}\Phi(\btheta;\alpha \bx_1 + (1-\alpha)\bx_2).
    \end{align*}
\end{proof}

We now turn to prove the lemmas.

\begin{proof}[Proof of Lemma~\ref{lma:unite_points_bias}]
    Denote $\lambda^\prime = \lambda_1 + \lambda_2$ and $\alpha = \frac{\lambda_1}{\lambda_1 + \lambda_2}$. It is easy to see that $\alpha \in (0,1)$ and that $\lambda^\prime > 0$, which proves that Equations~\ref{eq:positive_lambda}~and~\ref{eq:zero_lam} hold. Moreover, since $\bx_1$ and $\bx_2$ have the same activation pattern, we have that $\bx_{1.5}$ also has the same activation pattern since it is a convex combination of the two, which proves that \Eqref{eq:margin} holds. To show that \Eqref{eq:stationary} holds, we compute
    \begin{align*}
        \btheta &= \sum_{i=1}^n \lambda_i y_i \nabla_{\btheta}\Phi(\btheta; \bx_i) \\
        &= \sum_{i \neq 1,2}\lambda_i y_i \nabla_{\btheta}\Phi(\btheta; \bx_i) + \lambda_1 y_1 \nabla_{\btheta}\Phi(\btheta; \bx_1) + \lambda_2 y_2 \nabla_{\btheta}\Phi(\btheta; \bx_2) \\
        &\stackrel{*}{=} \sum_{i \neq 1,2}\lambda_i y_i \nabla_{\btheta}\Phi(\btheta; \bx_i) + \lambda_1 y_{1.5} \nabla_{\btheta}\Phi(\btheta; \bx_1) + \lambda_2 y_{1.5} \nabla_{\btheta}\Phi(\btheta; \bx_2) \\
        &\stackrel{**}{=} \sum_{i \neq 1,2}\lambda_i y_i \nabla_{\btheta}\Phi(\btheta; \bx_i) + y_{1.5} \nabla_{\btheta}\Phi(\btheta; \lambda_1 \bx_1 + \lambda_2 \bx_2)\\
        &= \sum_{i \neq 1,2}\lambda_i y_i \nabla_{\btheta}\Phi(\btheta; \bx_i) + y_{1.5} \nabla_{\btheta}\Phi\left(\btheta; (\lambda_1 + \lambda_2)\left(\frac{\lambda_1}{\lambda_1+\lambda_2} \bx_1 + \frac{\lambda_2}{\lambda_1 + \lambda_2} \bx_2\right)\right)\\
        &= \sum_{i \neq 1,2}\lambda_i y_i \nabla_{\btheta}\Phi(\btheta; \bx_i) + y_{1.5} \nabla_{\btheta}\Phi(\btheta; \lambda^\prime(\alpha \bx_1 + (1-\alpha) \bx_2))\\
        &= \sum_{i \neq 1,2}\lambda_i y_i \nabla_{\btheta}\Phi(\btheta; \bx_i) + y_{1.5} \lambda^\prime \nabla_{\btheta}\Phi(\btheta; \bx_{1.5}),
    \end{align*}
    where $*$ is by the assumption that $y_1 = y_2 = y_{1.5}$, and $**$ is by Lemma~\ref{lma:gradient-is-convex}.
\end{proof}

\begin{proof}[Proof of Lemma~\ref{lma:split_points_bias}]
    Denote $\lambda_\alpha = \frac{\alpha\lambda_1}{\alpha + \beta},~ \lambda_\beta = \frac{\beta \lambda_1}{\alpha + \beta}$. It is easy to see that $\lambda_\alpha, \lambda_\beta > 0$, which proves that Equations~\ref{eq:positive_lambda}~and~\ref{eq:zero_lam} hold. Moreover, since by assumption $\bz_1$ and $\bz_2$ are split such that they have the same activation pattern as $\bx_{1}$, this assures that \Eqref{eq:margin} holds. 
    Now, observe that
    \begin{enumerate}
        \item
        $\lambda_\beta + \lambda_\alpha = \frac{\lambda_1 \beta}{\alpha + \beta} + \frac{\lambda_1\alpha}{\alpha+\beta} = \lambda_1$,
        \item 
        $\lambda_\beta(\bx_1 + \alpha\bnu) + \lambda_\alpha(\bx_1 - \beta\bnu) = (\lambda_\beta + \lambda_\alpha)\bx_1 + \frac{\lambda_1 \beta}{\alpha+\beta} \alpha\bnu - \frac{\lambda_1\alpha}{\alpha+\beta}\beta\bnu = \lambda_1 \bx_1$.
    \end{enumerate}
    Then, by the above, we have
    \begin{align}
        \lambda_\beta\nabla_{\btheta}\Phi(\btheta; \bz_1) + \lambda_\alpha\nabla_{\btheta}\Phi(\btheta; \bz_2) &= \lambda_\beta(\bba(\bx_1 + \alpha\bnu) + \bc) + \lambda_\alpha(\bba(\bx_1 - \beta\bnu) + \bc) \nonumber\\
        &= (\lambda_\beta + \lambda_\alpha)\bba\bx_1 + (\lambda_\alpha + \lambda_\beta)\bc \nonumber\\
        &= \lambda_1(\bba\bx_1 + \bc) \nonumber\\
        &= \lambda_1\nabla_{\btheta}\Phi(\btheta; \bx_1). \label{eq:merged_lambdas}
    \end{align}
    Next, let us see that indeed \Eqref{eq:stationary} holds by computing
    \begin{align*}
        \btheta &= \sum_{i=1}^n \lambda_i y_i \nabla_{\btheta}\Phi(\btheta; \bx_i) \\
        &= \sum_{i \neq 1} \lambda_i y_i \nabla_{\btheta}\Phi(\btheta; \bx_i) + \lambda_1 y_1 \nabla_{\btheta}\Phi(\btheta; \bx_1) \\
        &\stackrel{*}{=} \sum_{i \neq 1} \lambda_i y_i \nabla_{\btheta}\Phi(\btheta; \bx_i) + y_1 \lambda_\beta\nabla_{\btheta}\Phi(\btheta; \bz_1) + y_1\lambda_\alpha\nabla_{\btheta}\Phi(\btheta; \bz_2),
    \end{align*}
    where $*$ is by \Eqref{eq:merged_lambdas}. Thus, the set $S^\prime$ is a KKT set.
\end{proof}

\subsection{Proofs of Theorem~\ref{thm:subspace} and Theorem~\ref{thm:kkt_maximal_distance}}

We begin with proving the following auxiliary lemmas. Lemma~\ref{lem:kkt_derivatives} provides applicable technical constraints that any KKT network must satisfy, and Lemma~\ref{lma:orthogonal_on_the_margin} shows that we can split the training samples in a single direction, and that all points in the new KKT set will remain on the margin. Lemma~\ref{lma:orthogonal_same_activation_pattern} assures that we can split along this direction as far as we want and remain in the same activation pattern as the original sample.

\begin{lemma}\label{lem:kkt_derivatives}
    Suppose a 2-layer, homogeneous ReLU neural network with parameters $\btheta$ that satisfies the KKT conditions in Equations~\ref{eq:stationary}-\ref{eq:zero_lam}. Then we have
    \[
        v_j = \sum_{i=1}^n\lambda_iy_i\relu{\bw_j^\top \bx_i+b_j},\quad
        \bw_j = v_j\sum_{i=1}^n\lambda_iy_i\bx_i\sigma'_{i,j},\quad
        b_j=v_j\sum_{i=1}^n\lambda_iy_i\sigma'_{i,j}.
    \]
\end{lemma}

\begin{proof}
    We compute the derivatives with respect to $\btheta$ as follows
    \begin{equation*}
       \frac{\partial}{\partial v_j}\Phi(\btheta;\bx) = \relu{\bw_j^\top \bx+b_j},\quad
       \frac{\partial}{\partial \bw_j}\Phi(\btheta;\bx) = v_jx\sigma'_j,\quad
       \frac{\partial}{\partial b_j}\Phi(\btheta;\bx) = v_j\sigma'_j.
    \end{equation*}
    If $\bw_j^\top \bx + b_j \neq 0$ then $\sigma'_{j}$ is well defined, and if $\bw_j^\top \bx + b_j = 0$ then $\sigma'_{j} \in [0,1]$. In any case, it holds that $\sigma'_{j} \geq 0$. Combining the above partial derivatives with \Eqref{eq:stationary}, we obtain
    \[
       v_j = \sum_{i=1}^n\lambda_iy_i\relu{\bw_j^\top \bx_i+b_j},\quad
       \bw_j = v_j\sum_{i=1}^n\lambda_iy_i\bx_i\sigma'_{i,j},\quad
       b_j=v_j\sum_{i=1}^n\lambda_iy_i\sigma'_{i,j}
    \]
    for all $j\in[k]$, as required.
\end{proof}

\begin{lemma}\label{lma:orthogonal_on_the_margin}
    Let $S = \{\bx_1,\dots,\bx_n\}$ be a KKT set such that all $\bx_i \in S$ are on the margin, let $p$ be the margin's value, and let $\hat{\bx}$ be a vector that is orthogonal to all $\bx_i \in S$. Then, for all $\beta \in \real$ and for all $l \in [n]$ we have that $|\Phi(\btheta; \bx_l + \beta\hat{\bx})| = p$.
\end{lemma}
\begin{proof}
    Compute
    \begin{align*}
    |\Phi(\btheta;\bx_l + \beta\hat{\bx})| &=\left|\sum_{j=1}^k v_j \relu{\langle \bw_j, \bx_l + \beta\hat{\bx} \rangle + b_j} \right|\stackrel{*}{=} \left|\sum_{j=1}^kv_j \relu{v_j \sum_{i=1}^n \lambda_i y_i \sigma^\prime_{i,j} \langle\bx_i, \bx_l + \beta\hat{\bx} \rangle + b_j}\right| \\
    &\stackrel{**}{=}\left|\sum_{j=1}^kv_j \relu{v_j \sum_{i=1}^n \lambda_i y_i \sigma^\prime_{i,j} \langle\bx_i, \bx_l\rangle + b_j} \right| = \left|\sum_{j=1}^kv_j \relu{\left\langle v_j \sum_{i=1}^n \lambda_i y_i \sigma^\prime_{i,j}\bx_i, \bx_l \right\rangle + b_j} \right|\\
    &= \left| \sum_{j=1}^k v_j \relu{\langle \bw_j, \bx_l \rangle + b_j} \right| =|\Phi(\btheta; \bx_l)|=p.
\end{align*}
Where $*$ is due to Lemma~\ref{lem:kkt_derivatives}, and $**$ is due to the fact that $\hat{\bx}$ is orthogonal to all $\bx_i$.
\end{proof}

\begin{lemma}\label{lma:orthogonal_same_activation_pattern}
    Let $S = \{\bx_1,\dots,\bx_n\}$ be a KKT set and let $\hat{\bx}$ be a vector that is orthogonal to all $\bx_i \in S$. Then, $c_j(\bx_l + \beta\hat{\bx}) = c_j(\bx_l)$ for all $l \in [n]$ and for all $\beta \in \real$.
\end{lemma}
\begin{proof}
    Compute
    \begin{align*}
    c_j(\bx_l + \beta \hat{\bx}) &= \relu{\langle \bw_j, \bx_l + \beta\hat{\bx} \rangle + b_j} \stackrel{*}{=} \relu{v_j \sum_{i=1}^n \lambda_i y_i \sigma^\prime_{i,j} \langle \bx_i, \bx_l + \beta\hat{\bx} \rangle + b_j} \\
    &\stackrel{**}{=} \relu{v_j \sum_{i=1}^n \lambda_i y_i \sigma^\prime_{i,j} \langle \bx_i,\bx_l\rangle + b_j} \stackrel{***}{=} \relu{\langle \bw_j, \bx_l \rangle + b_j} = c_j(\bx_l),
\end{align*}
where $*$ and $***$ use Lemma~\ref{lem:kkt_derivatives}, and $**$ uses that fact that $\hat{\bx}$ is orthogonal to all $\bx_i$.
\end{proof}

Using these two lemmas, we can now prove Theorem~\ref{thm:subspace}.
\begin{proof}[Proof of Theorem~\ref{thm:subspace}]
    Let $S = \{\bx_1, \dots, \bx_n\}$ be a KKT set. Since $\spn \{\bx_1,\dots,\bx_n\} \subsetneq \real^d$, there exists a vector $\hat{\bx} \in \real^d$ that is orthogonal to all $\bx_l \in S$. Since $\hat{\bx}$ is a directional vector, we can assume without loss of generality that $\| \hat{\bx} \| = 1$. For each $\bx_l \in S$ we define two new points $\bx_{l_1} = \bx_l + \alpha_l \hat{\bx}$ and $\bx_{l_2} = \bx_l - \beta_l \hat{\bx}$ for some $\alpha_l, \beta_l > 0$. We need to show that the set $S^\prime = \bigcup_{l=1}^n \{\bx_{l_1}, \bx_{l_2}\}$ is a KKT set.
    Using Lemma~\ref{lma:split_points_bias} (which we can use since we know that $\bx_l$ and $\{\bx_{l_1}, \bx_{l_2}\}$ have the same activation pattern from Lemma~\ref{lma:orthogonal_same_activation_pattern}, and also the same classification by assumption), we can replace each instance $\bx_l \in S$ with $\bx_{l_1}, \bx_{l_2}$ iteratively until we get $S^\prime$. This proves that $S^\prime$ is a KKT set. Moreover, by Lemma~\ref{lma:orthogonal_on_the_margin} we have that all points in $S^\prime$ are on the margin. Lastly, we need to prove that for any distance $\tau > 0$ we can chose $\alpha_l, \beta_l$ such that $d(S, S^\prime) > \tau$. Let us compute the distance of some $\bx_{l_1} = \bx_l + \alpha_l \hat{\bx}$ from all points $\bx_i \in S$ as follows
    \begin{align*}
        \|\bx_{l_1} - \bx_i\|_2^2 &= \|\bx_{l_1}\|^2 - 2 \langle \bx_{l_1}, \bx_i \rangle + \|\bx_i\|^2 \\ 
        &= \|\bx_l + \alpha_l \hat{\bx}\|^2 - 2 \langle \bx_l + \alpha_l \hat{\bx}, \bx_i \rangle + \|\bx_i\|^2 \\
        &= \|\bx_l + \alpha_l \hat{\bx}\|^2 - 2 \langle \bx_l, \bx_i \rangle + \|\bx_i\|^2\\
        &= \| \bx_l \|^2 + 2 \alpha_l \langle \bx_l, \hat{\bx} \rangle + \alpha_l^2 \|\hat{\bx} \|^2 - 2 \langle \bx_l, \bx_i \rangle + \|\bx_i\|^2\\
        &= \| \bx_l \|^2 + \alpha_l^2 \|\hat{\bx} \|^2 - 2 \langle \bx_l, \bx_i \rangle + \|\bx_i\|^2\\
        &= \| \bx_l \|^2 + \alpha_l^2 - 2 \langle \bx_l, \bx_i \rangle + \|\bx_i\|^2.
    \end{align*}
    Since $\bx_l$ and $\bx_i$ are constants, then $\| \bx_l \|^2$, $2 \langle \bx_l, \bx_i \rangle$ and $\|\bx_i\|^2$ are also constants. The only parameter we can change is $\alpha_l$, so we can take $\alpha_l$ large enough to make sure that $d(\bx_{l_1}, \bx_i) > \tau$ for all $\bx_i \in S$. The same analysis also shows that this approach can be applied iteratively to all remaining instances. Note that we can take $\alpha_l$ and $\beta_l$ to be as large as we want since the new points $\bx_{l_1}$ and $\bx_{l_2}$ will still have the same activation pattern as $\bx_l$ (for all $\bx_l \in S$) by virtue of Lemma~\ref{lma:orthogonal_same_activation_pattern}. We can thus bound the distance between $S$ and $S^\prime)$ by
    \begin{align*}
        d(S,S^\prime) = \min_{\bx \in S, \bx^\prime \in S^\prime} \|\bx - \bx^\prime\| > \tau.
    \end{align*}
\end{proof}

We now turn to proving Theorem~\ref{thm:kkt_maximal_distance}.
\begin{proof}[Proof of Theorem~\ref{thm:kkt_maximal_distance}]
    Let $\bx_l \in S$ and let $\bx_l + \alpha\bnu$ and $\bx_l - \beta \bnu$ be the new points resulted from splitting $\bx_l$. Both points have to remain in the same activation pattern as $\bx_l$, meaning that for each neuron $c_j(\bx)$, we need to make sure that $\sign(c_j(\bx_l)) = \sign(c_j(\bx_l + \alpha\bnu)) = \sign(c_j(\bx_l - \beta\bnu))$. Namely, that we do not change the ReLU activation of any neuron $\bw_j^\top \bx + b_j$. $\bx_l + \alpha \bnu$ changes an activation when $\langle \bw_j, \bx_l + \alpha \bnu \rangle + b_j = 0$, implying that
    \begin{align*}
        \langle \bw_j, \bx_l + \alpha \bnu \rangle + b_j  = 0 &\Rightarrow \langle\bw_j, \bx_l \rangle + b_j + \alpha \langle \bw_j, \bnu \rangle = 0 \\
        &\Rightarrow \alpha = - \frac{\langle \bw_j, \bx_l \rangle + b_j}{\langle \bw_j, \bnu \rangle}.
    \end{align*}
    Now we can upper bound the magnitude of $\alpha$ to see how far away $\bx_l + \alpha \bnu$ can be from $\bx_l$. Let us denote the signed distance between a point $\bx$ and the hyperplane $\langle \bw_j, \bx \rangle + b_j = 0$ by $D_j(\bx) = \frac{\langle \bw_j, \bx \rangle + b_j}{\|\bw_j\|}$. We can rewrite $\alpha = -\frac{D_j(\bx_l) \|\bw_j\|}{\langle \bw_j, \bnu \rangle}$. Now let us bound $|\langle \bw_j, \bnu \rangle|$ as follows
    \begin{align*}
        |\langle \bw_j, \bnu \rangle| &= \left|\left \langle \sum_{i=1}^n \lambda_i y_i \sigma^\prime_{i,j} \bx_i, \bnu \right \rangle \right| \\
        &=\left|\sum_{i=1}^n \lambda_i y_i \sigma^\prime_{i,j} \langle\bx_i, \bnu \rangle \right| \\
        &\leq \sum_{i=1}^n \lambda_i |y_i| |\sigma^\prime_{i,j}| |\langle\bx_i, \bnu \rangle| \\
        &\leq \gamma \sum_{i=1}^n \lambda_i.
        \end{align*}
    The above implies that
    \[
        \alpha \geq \frac{|D_j(\bx_l)| \cdot \|\bw_j\|}{\gamma \sum_{i=1}^n \lambda_i} = \frac{|D_j(\bx_l)| \cdot \|\bw_j\|}{\sum_{i=1}^n \lambda_i} \cdot \frac{1}{\gamma}.
    \]
    This is a sufficient condition on all neurons $c_j$ to guarantee that we do not deviate from our current activation pattern. This means that any $\alpha$ not exceeding
    \[
        \min_{j \in [k]} \frac{|D_j(\bx_l)| \cdot \|\bw_j\|}{\sum_{i=1}^n \lambda_i} \cdot \frac{1}{\gamma}
    \]
    is a valid choice, where the exact same analysis gives the same bound for $\beta$.
\end{proof}

\section{Proofs for Subsection~\ref{sec:almost-kkt}}\label{sec:almost_kkt_proofs}
\subsection{Proofs of Lemma~\ref{lma:unite_points_bias_almost_kkt} and Lemma~\ref{lma:split_points_bias_almost_kkt}} \label{sec:almost_kkt_lemmas_proofs}

\begin{proof}[Proof of Lemma~\ref{lma:unite_points_bias_almost_kkt}]
    Define $\lambda^\prime \coloneqq \lambda_1 + \lambda_2$ and $\alpha \coloneqq \frac{\lambda_1}{\lambda_1 + \lambda_2}$, $\bx_{1.5} \coloneqq \alpha\bx_1 + (1-\alpha)\bx_2$. It is easy to see that $\alpha \in (0,1)$ and $\lambda^\prime > 0$, which proves that Condition~\ref{eq:almost_positive_lambdas} holds. Moreover, Condition~\ref{eq:delta_kkt} is also easily satisfied for a sufficiently large $\delta$, and Condition~\ref{eq:almost_margin} holds since the merged point $\bx_{1.5}$ is a convex combination of $\bx_1,\bx_2$ and thus has the same activation pattern and classification as them. To show Condition~\ref{eq:epsilon_kkt}, we compute
    \begin{align*}
        \sum_{i=1}^n \lambda_i y_i \nabla_{\btheta}\Phi(\btheta; \bx_i) &= \sum_{i \neq 1,2}\lambda_i y_i \nabla_{\btheta}\Phi(\btheta; \bx_i) + \lambda_1 y_1 \nabla_{\btheta}\Phi(\btheta; \bx_1) + \lambda_2 y_2 \nabla_{\btheta}\Phi(\btheta; \bx_2) \\
        &\stackrel{*}{=} \sum_{i \neq 1,2}\lambda_i y_i \nabla_{\btheta}\Phi(\btheta; \bx_i) + \lambda_1 y_{1.5} \nabla_{\btheta}\Phi(\btheta; \bx_1) + \lambda_2 y_{1.5} \nabla_{\btheta}\Phi(\btheta; \bx_2) \\
        &\stackrel{**}{=} \sum_{i \neq 1,2}\lambda_i y_i \nabla_{\btheta}\Phi(\btheta; \bx_i) + y_{1.5} \nabla_{\btheta}\Phi(\btheta; \lambda_1 \bx_1 + \lambda_2 \bx_2)\\
        &= \sum_{i \neq 1,2}\lambda_i y_i \nabla_{\btheta}\Phi(\btheta; \bx_i) + y_{1.5} \nabla_{\btheta}\Phi\left(\btheta; (\lambda_1 + \lambda_2)\left(\frac{\lambda_1}{\lambda_1+\lambda_2} \bx_1 + \frac{\lambda_2}{\lambda_1 + \lambda_2} \bx_2\right)\right)\\
        &= \sum_{i \neq 1,2}\lambda_i y_i \nabla_{\btheta}\Phi(\btheta; \bx_i) + y_{1.5} \nabla_{\btheta}\Phi(\btheta; \lambda^\prime(\alpha \bx_1 + (1-\alpha) \bx_2))\\
        &= \sum_{i \neq 1,2}\lambda_i y_i \nabla_{\btheta}\Phi(\btheta; \bx_i) + y_{1.5} \lambda^\prime \nabla_{\btheta}\Phi(\btheta; \bx_{1.5}),
    \end{align*}
    where $*$ is from the assumption that $y_1 = y_2 = y_{1.5}$, and $**$ follows from Lemma~\ref{lma:gradient-is-convex}, implying that
    \begin{align*}
        \|\btheta - \sum_{i \neq 1,2}\lambda_i y_i \nabla_{\btheta}\Phi(\btheta; \bx_i) - y_{1.5} \lambda^\prime \nabla_{\btheta}\Phi(\btheta; \bx_{1.5})\|^2 = \|\btheta - \sum_{i=1}^n \lambda_i y_i \nabla_{\btheta}\Phi(\btheta; \bx_i)\|^2 \leq \varepsilon,
    \end{align*}
    and therefore $S^\prime$ is also an $\varepsilon$-KKT set.
\end{proof}

\begin{proof}[Proof of Lemma~\ref{lma:split_points_bias_almost_kkt}]
    Denote $\lambda_\alpha = \frac{\alpha\lambda_1}{\alpha + \beta},~ \lambda_\beta = \frac{\beta \lambda_1}{\alpha + \beta}$. It is easy to see that $\lambda_\alpha, \lambda_\beta > 0$, which proves that Condition~\ref{eq:almost_positive_lambdas} holds. Moreover, Condition~\ref{eq:delta_kkt} is also easily satisfied for a sufficiently large $\delta$, and Condition~\ref{eq:almost_margin} holds by the assumption that the predictions of $\Phi$ on $\bz_1$ and $\bz_2$ change by less than the margin value $p$. To show Condition~\ref{eq:epsilon_kkt}, observe that since $\Phi(\btheta; \bx)$ is piecewise linear in $\bx$, $\nabla_{\btheta}\Phi(\btheta; \bx)$ is affine in $\bx$ inside each activation pattern, and can be written as $\nabla_{\btheta}\Phi(\btheta; \bx) = \bba \bx + \bc$ for some matrix $\bba$ and vector $\bc$. This implies that
    \begin{enumerate}
        \item
        $\lambda_\beta + \lambda_\alpha = \frac{\lambda_1 \beta}{\alpha + \beta} + \frac{\lambda_1\alpha}{\alpha+\beta} = \lambda_1$,
        \item 
        $\lambda_\beta(\bx_1 + \alpha\bnu) + \lambda_\alpha(\bx_1 - \beta\bnu) = (\lambda_\beta + \lambda_\alpha)\bx_1 + \frac{\lambda_1 \beta}{\alpha+\beta} \alpha\bnu - \frac{\lambda_1\alpha}{\alpha+\beta}\beta\bnu = \lambda_1 \bx_1$,
    \end{enumerate}
    which is turn shows that,
    \begin{align*}
        \lambda_\beta\nabla_{\btheta}\Phi(\btheta; \bz_1) + \lambda_\alpha\nabla_{\btheta}\Phi(\btheta; \bz_2) &= \lambda_\beta(\bba(\bx_1 + \alpha\bnu) + \bc) + \lambda_\alpha(\bba(\bx_1 - \beta\bnu) + \bc) \\
        &= (\lambda_\beta + \lambda_\alpha)\bba\bx_1 + (\lambda_\alpha + \lambda_\beta)\bc \\
        &= \lambda_1(\bba\bx_1 + \bc) \\
        &= \lambda_1\nabla_{\btheta}\Phi(\btheta; \bx_1).
    \end{align*}
    The above implies
    \begin{align*}
        \|\btheta - \sum_{i \neq 1}\lambda_i y_i\nabla_{\btheta}\Phi(\btheta;\bx_i) - \lambda_\beta\nabla_{\btheta}\Phi(\btheta; \bz_1) -\lambda_\alpha\nabla_{\btheta}\Phi(\btheta; \bz_2)\|^2 = \|\btheta - \sum_{i =1}^n\lambda_i y_i\nabla_{\btheta}\Phi(\btheta;\bx_i)\|^2 \leq \varepsilon,
    \end{align*}
    hence, $S^\prime$ is also an $\varepsilon$-KKT set.
\end{proof}

\subsection{Splitting distance lower bounds}\label{sec:almost-kkt-distance}
In this subsection of the appendix, we establish a bound on the distance between the original point and the new points obtained through the splitting procedure when using Lemma~\ref{lma:split_points_bias_almost_kkt} in the almost-KKT setting. The following lemma will be useful to bound the deviations in the predictions of the model that result from splitting a data instance in a direction $\bnu$.

\begin{lemma}\label{lem:dot_product_bound}
    Given an $(\varepsilon,\delta)$-KKT point $\btheta$, suppose that $\bnu$ is a directional vector satisfying $\norm{\bnu}=1$ and $|\langle\bx_i,\bnu\rangle|\le\gamma$ for all $i$, for some $\gamma>0$. Then, we have for any neuron $j$ that
    \[
        | \langle \bw_j, \bnu \rangle| \le \varepsilon + \gamma |v_j| \sum_{i=1}^n \lambda_i.
    \]
\end{lemma}

\begin{proof}
    Compute
    \begin{align*}
        | \langle \bw_j, \bnu \rangle| &= \left | \langle \bw_j, \bnu \rangle - \left \langle v_j \sum_{i=1}^n \lambda_i y_i \sigma^\prime_{i,j} \bx_i, \bnu\right\rangle + \left \langle v_j \sum_{i=1}^n \lambda_i y_i \sigma^\prime_{i,j} \bx_i, \bnu\right\rangle \right | \\
        &\leq \left | \left \langle \bw_j - v_j \sum_{i=1}^n \lambda_i y_i \sigma^\prime_{i,j} \bx_i, \bnu \right \rangle \right | + \left | v_j \sum_{i=1}^n \lambda_i y_i \sigma^\prime_{i,j} \langle \bx_i, \bnu \rangle \right | \\
        &\leq \left \| \bw_j - v_j \sum_{i=1}^n \lambda_i y_i \sigma^\prime_{i,j} \bx_i \right \| \cdot \|\bnu \| + |v_j| \sum_{i=1}^n |\lambda_i y_i \sigma^\prime_{i,j}| \cdot | \langle \bx_i, \bnu \rangle| \\
        &= \left \| \bw_j - v_j \sum_{i=1}^n \lambda_i y_i \sigma^\prime_{i,j} \bx_i \right \| \cdot \|\bnu \| + |v_j| \sum_{i=1}^n \lambda_i \sigma^\prime_{i,j} \cdot | \langle \bx_i, \bnu \rangle| \\
        &\leq \varepsilon + \gamma |v_j| \sum_{i=1}^n \lambda_i.
    \end{align*}
\end{proof}

Equipped with the above lemma, we now turn to prove the theorems.

\begin{proof}[Proof of Theorem~\ref{thm:distance_almost_kkt}]
    The theorem follows from Lemma~\ref{lma:split_points_bias_almost_kkt}, as long as $\bx_l + \alpha_l\bnu$ has the same activation pattern as $\bx_l$. To this end, we need to bound $|\frac{\langle \bw_j, \bx_l \rangle + b_j}{\langle \bw_j, \bnu \rangle}|$, which is the distance between $\bx_l$ and the hyperplane induced by $c_j(\bx)$, for each neuron $c_j(\bx)$.

    Denoting the signed distance between a point $\bx$ and the hyperplane $\langle \bw_j, \bx \rangle + b_j = 0$ by $D_j(\bx) = \frac{\langle \bw_j, \bx \rangle + b_j}{\|\bw_j\|}$, we can rewrite 
    \[
        |\alpha_l| \ge \left|\frac{D_j(\bx_l) \|\bw_j\|}{\langle \bw_j, \bnu \rangle}\right| \geq \frac{|D_j(\bx_l)| \|\bw_j\|}{\varepsilon + \gamma |v_j| \sum_{i=1}^n\lambda_i},
    \]
    where the last inequality follows from Lemma~\ref{lem:dot_product_bound}. Namely, we can take $|\alpha_l|$ as small as 
    \[
        \min_{j \in [n]} \frac{|D_j(\bx_l)| \|\bw_j\|}{\varepsilon + \gamma |v_j| \sum_{i=1}^n\lambda_i}.
    \]
    The same reasoning yields the same bound for $|\beta_l|$.
\end{proof}

\begin{proof}[Proof of Theorem~\ref{thm:delta_almost_kkt}]
First, let us bound $\left| (\Phi(\btheta; \bx_l + \alpha_l \bnu) - \Phi(\btheta; \bx_l))) \right|$ as follows
    \begin{align}
        |\Phi(\btheta; \bx_l) - \Phi(\btheta; \bx_l + \alpha_l \bnu)| &= \left | \sum_{j \in J} v_j \relu{\langle\bw_j, \bx_l \rangle + b_j} -  \sum_{j \in J} v_j \relu{\langle\bw_j, \bx_l + \alpha_l \bnu \rangle + b_j} \right | \nonumber\\
        &=\left| \sum_{j \in J} v_j \left( \relu{\langle\bw_j, \bx_l \rangle + b_j}  - \relu{\langle\bw_j, \bx_l + \alpha_l \bnu \rangle + b_j} \right) \right| \nonumber\\
        &\leq \sum_{j \in J} |v_j| \left| \relu{\langle\bw_j, \bx_l \rangle + b_j}  - \relu{\langle\bw_j, \bx_l + \alpha_l \bnu \rangle + b_j} \right| \nonumber\\
        &\stackrel{*}{\leq} \sum_{j \in J} |v_j| |\langle\bw_j, \alpha_l \bnu \rangle| = \alpha_l \sum_{j \in J} |v_j| |\langle\bw_j, \bnu \rangle| \nonumber\\
        &\stackrel{**}{\leq} \alpha_l \sum_{j \in J} |v_j|\left(\varepsilon + \gamma |v_j| \sum_{i=1}^n \lambda_i \right), \label{eq:bound_network_output}
    \end{align}
where $*$ is by the fact that $\relu{\cdot}$ is $1$-Lipschitz, and $**$ is by Lemma~\ref{lem:dot_product_bound}. If the above quantity does not deviate beyond the value of the margin $p$, then we can guarantee that Condition~\ref{eq:almost_margin} still holds and we can use Lemma~\ref{lma:split_points_bias_almost_kkt}. It now only remains to bound by how much $\delta$ deteriorates as follows
    \begin{align*}
        &\left|\frac{\lambda_l \beta_l}{\alpha_l + \beta_l}(y_l \cdot \Phi(\btheta; \bx_l + \alpha_l \bnu) - p)\right| = \left|\frac{\lambda_l \beta_l}{\alpha_l + \beta_l} (y_l \cdot \Phi(\btheta; \bx_l + \alpha_l \bnu) - y_l \cdot \Phi(\btheta; \bx_l) + y_l \cdot \Phi(\btheta; \bx_l) - p)\right| \\
        &\hskip 2cm= \left|\frac{\lambda_l \beta_l}{\alpha_l + \beta_l} (y_l \Phi(\btheta;\bx_l) - p) + \frac{\lambda_l \beta_l}{\alpha_l + \beta_l} (y_l (\Phi(\btheta; \bx_l + \alpha_l \bnu) - \Phi(\btheta; \bx_l))) \right| \\
        &\hskip 2cm\leq \left | \frac{\lambda_l \beta_l}{\alpha_l + \beta_l} (y_l \Phi(\btheta;\bx_l) - p) \right | + \left | \frac{\lambda_l \beta_l}{\alpha_l + \beta_l} (y_l (\Phi(\btheta; \bx_l + \alpha_l \bnu) - \Phi(\btheta; \bx_l))) \right | \\
        &\hskip 2cm\leq \frac{\beta_l}{\alpha_l + \beta_l} \delta + \frac{\lambda_l \beta_l}{\alpha_l + \beta_l} |y_l| \left| (\Phi(\btheta; \bx_l + \alpha_l \bnu) - \Phi(\btheta; \bx_l))) \right| \\
        &\hskip 2cm= \frac{\beta_l}{\alpha_l + \beta_l} \delta + \frac{\lambda_l \beta_l}{\alpha_l + \beta_l} \left| (\Phi(\btheta; \bx_l + \alpha_l \bnu) - \Phi(\btheta; \bx_l))) \right| \\
        &\hskip 2cm\leq \delta + \lambda_l \alpha_l \sum_{j \in J} |v_j| \left(\varepsilon + \gamma |v_j| \sum_{i=1}^n \lambda_i \right),
    \end{align*}
    where the last inequality is due to \Eqref{eq:bound_network_output} and the fact that $\frac{\beta_l}{\alpha_l + \beta_l}\le1$. The theorem then follows by summing the above bound with an analogous bound obtained for $\beta_l$.
\end{proof}

\section{Finding an almost orthogonal splitting direction using SVD decomposition}\label{app:SVD}

Theorems~\ref{thm:distance_almost_kkt}~and~\ref{thm:delta_almost_kkt} require finding a direction $\bnu$ such that $\langle\bnu,\bx_i\rangle$ for all data instances $\bx_i$ to be effective. Conveniently, such an upper bound can be derived in terms of the smallest singular value of the data matrix. Given a training set $S$, we write it as a column matrix $\bbs = (\bx_1, \dots, \bx_n)$. Then we have the following upper bound on the dot products.

\begin{theorem}\label{lma:svd}
    Let $\bbs = \bbu \bSigma \bbv^\top$ be the SVD decomposition of the training data matrix $\bbs$, where $\bbu \in \real^{d \times d}, \bbv \in \real^{n \times n}$, and $\bSigma \in \real^{d \times n}$ is a diagonal matrix with entries $\sigma_1 \geq \ldots\geq \sigma_d \geq 0 $. Then, there exists a unit vector $\bnu$ such that $|\langle \bx_i, \bnu \rangle| \leq \sigma_d$ for all $i \in [n]$.
\end{theorem}
\begin{proof}
    Let us chose $\bnu = \bbu_d$, i.e the $d$-th column in $\bbu$. Let us denote $\bbx_i$ to be the $i$-th column. We need to show that $|\langle\bbx_i, \bnu \rangle| \leq \sigma_d$, or equivalently $|\bbx_i^\top \bnu| \leq \sigma_d$. We can write $\bbx_i$ as $\bbx_i = \bbu \bSigma \bbv^\top_i = \sum_{j=1}^d \sigma_j \bbv^\top_{i,j} \bbu_j$ where $\bbv^\top_{i,j}$ is the $j$-th entry of the vector $\bbv_i$. Compute
    \begin{align*}
        |\langle \bx_i, \bnu \rangle| &= \left|\left \langle \sum_{j=1}^d\sigma_i \bbv^\top_{i,j} \bbu_i, \bbu_d \right \rangle \right| = \left| \sum_{j=1}^d\sigma_i \bbv^\top_{i,j} \langle \bbu_i, \bbu_d \rangle \right| \stackrel{*}{=} |\sigma_d \bbv_{i,d}^\top| \stackrel{**}{\leq} \sigma_d,
    \end{align*}
    where $*$ is due to the fact that $\bbu$ is orthogonal and $**$ is due to the fact that $\bbv$ is also orthogonal, and $\sigma_d \geq 0$.
\end{proof}

We remark that the above immediately implies that $\gamma\le\sigma_d$ for $\gamma$ in the statements of Theorems~\ref{thm:distance_almost_kkt}~and~\ref{thm:delta_almost_kkt}. Note that the final inequality may be rather loose, since $\bbv_{d}^\top$ is a unit vector; equality holds if and only if the entry $\bbv_{i,d}^\top$ equals $1$. If, however, the vector $\bbv_{i,d}^\top$ has more equally distributed entries, a tighter upper bound with a magnitude of roughly $\sigma_d/n$ will hold, further amplifying the efficacy of our theorems.

\section{Shifting the training set post-training}\label{app:shifting}

In this appendix, we provide additional details on how we obtained the trained network for our CIFAR experiment in Section~\ref{sec:experiments}. Recall that our goal was to shift the training data by a constant bias and then retrain the architecture used by \citet{haim2022reconstructing} on the shifted data, to obscure the attacker's prior knowledge. Since shifting the training set by a constant bias shifts the gradients of the resulting objective function by the same bias, shifting both the data and the initialization point leads to convergence to the same neural network as obtained by shifting the network trained on the original, unshifted data.

In light of this, rather than shifting the data and retraining the network from scratch, we utilized the network trained by \citet{haim2022reconstructing} and adjusted its biases in the first hidden layer. As an example, consider a hidden neuron with weights $\bw$ and bias $b$. Given an input $\bx$, the pre-activation output of this neuron is $\langle\bw,\bx\rangle+b$. Suppose we wish to shift each coordinate of the input $\bx$ by a vector $\bu$. In this case, the pre-activation output of the shifted input is
\[
    \langle\bw,\bx+\bu\rangle+b=\langle\bw,\bx\rangle+\langle\bw,\bu\rangle+b.
\]
Since $\langle\bw,\bu\rangle$ is constant, by modifying the bias term via the transformation $b\mapsto b-\langle\bw,\bu\rangle$, we recover the behavior of the original network on the shifted input. Applying this transformation to all biases in the first hidden layer for various values of $\bu$, we obtained the networks used in Figure~\ref{fig:move_pixels}.

\end{document}